\documentclass[11pt]{article}
\usepackage[utf8]{inputenc}
\usepackage[margin=1in]{geometry}

\usepackage[utf8]{inputenc} %
\usepackage[T1]{fontenc}    %
\usepackage{hyperref}       %
\usepackage{url}            %
\usepackage{booktabs}       %
\usepackage{amsfonts}       %
\usepackage{nicefrac}       %
\usepackage{microtype} %
\usepackage{natbib}
\usepackage{xcolor}         %
\usepackage{amsfonts}
\usepackage{amsthm}
\usepackage{algorithm,algorithmic}
\usepackage{amsmath}
\usepackage{amssymb}
\usepackage[vvarbb]{newtxmath}
\usepackage{mathtools}

\usepackage{etoolbox}
\usepackage{amssymb}

\usepackage{nicematrix}
\NiceMatrixOptions{left-margin,right-margin}

\usepackage{verbatim}
\usepackage{tabto}

\usepackage{tkdefs}

\usepackage{hyperref}
\usepackage{enumitem}
\usepackage{thmtools}
\usepackage{fancyhdr}
\usepackage{xspace}
\usepackage{xcolor}
\usepackage{xparse}
\newcommand{\spike}[2]%
{\bgroup
  \sbox0{#2}%
  \rlap{\usebox0}%
  \hspace{0.5\wd0}%
  \makebox[0pt][c]{\rule[\dimexpr \ht0+1pt]{0.5pt}{#1}}%
  \makebox[0pt][c]{\rule[\dimexpr -\dp0-#1-1pt]{0.5pt}{#1}}%
  \hspace{0.5\wd0}%
\egroup}
\usepackage[font=small]{caption}

\usepackage{booktabs}
\usepackage{array}
\usepackage{arydshln}
\usepackage[extdef=true]{delimset}
\usepackage[capitalize]{cleveref}
\RequirePackage{amsthm}
\RequirePackage{thmtools}
\RequirePackage[capitalize]{cleveref}

\RequirePackage{crossreftools}

\newcommand{\D}{\mathcal{D}}

\newcommand{\predclass}{\F_{\rowsnum,\colsnum}^{\ell,W_0}}

\newcommand{\F}{\mathcal{F}}
\newcommand{\Z}{\mathcal{Z}}
\newcommand{\N}{\mathbb{N}}

\renewcommand{\epsilon}{\varepsilon}

\newcommand{\vecentry}[2]{#1[#2]}
\newcommand{\vecentries}[3]{\vecentry{#1}{#2:#3}}

\newcommand{\erm}{\widehat{w}_*}

\newcommand{\vecpart}[2]{{#1}^{(#2)}}

\newcommand{\opt}{w_*}
\newcommand{\R}{\mathbb{R}}
\newcommand{\E}{\mathbb{E}}
\newcommand{\convdim}{d}
\newcommand{\rowsnum}{k}
\newcommand{\numdata}{n}
\newcommand{\colsnum}{m}
\newcommand{\unitball}[1]{\mathbb{B}^{#1}}
\newcommand{\frobball}[2]{\mathbb{B}^{{#1}{\times}{#2}}}
\newcommand{\unitballconv}{\unitball{\convdim}}
\newcommand{\unitballrows}{\unitball{\rowsnum}}
\newcommand{\unitballcols}{\unitball{\colsnum}}
\newcommand{\unitballrowscols}{\frobball{\rowsnum}{\colsnum}_{\initmatrix}}
\newcommand{\frobnorm}[1]{\|#1\|_F}
\newcommand{\optmat}{W^*}
\newcommand{\ermmat}{\widehat{W}_*}
\newcommand{\initmatrix}{W_{0}}

\declaretheoremstyle[
	    spaceabove=\topsep, 
	    spacebelow=\topsep, 
	    headfont=\normalfont\bfseries,
	    bodyfont=\normalfont\itshape,
	    notefont=\normalfont\bfseries,
	    notebraces={(}{)},
	    postheadspace=0.33em, 
	    headpunct={.},
    ]{theorem}
\declaretheorem[style=theorem]{theorem}

\declaretheoremstyle[
	    spaceabove=\topsep, 
	    spacebelow=\topsep, 
	    headfont=\normalfont\bfseries,
	    bodyfont=\normalfont,
	    notefont=\normalfont\bfseries,
	    notebraces={(}{)},
	    postheadspace=0.33em, 
	    headpunct={.},
    ]{definition}

\declaretheoremstyle[
        spaceabove=\topsep, 
        spacebelow=\topsep, 
        headfont=\normalfont\bfseries,
        bodyfont=\normalfont,
        notefont=\normalfont\bfseries,
        notebraces={}{},
        postheadspace=0.33em, 
        qed=$\blacksquare$, 
        headpunct={.},
    ]{proofstyle}
\declaretheorem[style=proofstyle,numbered=no,name=Proof]{proof}

\declaretheorem[style=theorem,name=Lemma]{lemma}

\declaretheorem[style=theorem,numbered=no,name=Lemma]{lemma*}
\declaretheorem[style=theorem,numbered=no,name=Corollary]{corollary*}
\declaretheorem[style=theorem,numbered=no,name=Proposition]{proposition*}
\declaretheorem[style=theorem,numbered=no,name=Claim]{claim*}
\declaretheorem[style=theorem,numbered=no,name=Fact]{fact*}
\declaretheorem[style=theorem,numbered=no,name=Observation]{observation*}
\declaretheorem[style=theorem,numbered=no,name=Conjecture]{conjecture*}

\declaretheorem[style=definition,numbered=no,name=Definition]{definition*}
\declaretheorem[style=definition,numbered=no,name=Assumption]{assumption*}
\declaretheorem[style=definition,numbered=no,name=Remark]{remark*}
\declaretheorem[style=definition,numbered=no,name=Example]{example*}
\declaretheorem[style=definition,numbered=no,name=Question]{question*}

\title{Complexity of Vector-valued Prediction:
From Linear Models to Stochastic Convex Optimization}

\author{%
    Matan Schliserman\thanks{Blavatnik School of Computer Science, Tel Aviv University; \texttt{schliserman@mail.tau.ac.il}.}
    \and%
    Tomer Koren\thanks{Blavatnik School of Computer Science, Tel Aviv University, and Google Research; \texttt{tkoren@tauex.tau.ac.il}.}
}

\begin{document}
\maketitle

\begin{abstract}
We study the problem of learning vector-valued linear predictors: these are prediction rules parameterized by a matrix that maps an $\colsnum$-dimensional feature vector to a $\rowsnum$-dimensional target. 
We focus on the fundamental case with a convex and Lipschitz loss function, and show several new theoretical results that shed light on the complexity of this problem and its connection to related learning models. 
First, we give a tight characterization of the sample complexity of Empirical Risk Minimization (ERM) in this setting, establishing that $\smash{\widetilde{\Omega}}(\rowsnum/\epsilon^2)$ examples are necessary for ERM to reach $\epsilon$ excess (population) risk; this provides for an exponential improvement over recent results by \cite{magen2023initialization} in terms of the dependence on the target dimension $\rowsnum$, and matches a classical upper bound due to \citet{maurer2016vector}.
Second, we present a black-box conversion from general $\convdim$-dimensional Stochastic Convex Optimization (SCO) to vector-valued linear prediction, showing that any SCO problem can be embedded as a prediction problem with $\rowsnum=\Theta(\convdim)$ outputs.
These results portray the setting of vector-valued linear prediction as bridging between two extensively studied yet disparate learning models: linear models (corresponds to $\rowsnum=1$) and general $\convdim$-dimensional SCO (with $\rowsnum=\Theta(\convdim)$).

\end{abstract}

\section{Introduction}

Prediction problems, such as classification and regression, lie at the core of both practical applications and theoretical research in machine learning. Within this framework, learning vector-valued predictors~(VVPs), characterized by functions of the form:
\[
x~\rightarrow~\ell(Wx)~,
\]
mapping vectors $x \in \R^\colsnum$ through a linear transformation parameterized by a matrix $W\in \R^{\rowsnum\times\colsnum}$ followed by a loss function $\ell : \R^k \mapsto \R$, constitutes a rich learning framework that captures a wide range of problems in machine learning, from classical to modern.
For instance, the scenario where $\rowsnum=1$ corresponds to the extensively studied domain of generalized linear models~\cite[e.g.,][]{bartlett2002rademacher,DBLP:books/daglib/shalev-shwartz2014}. 
When $\rowsnum > 1$, this setting encompasses multi-class problems, where $W$ acts as a matrix of predictors and $\ell$ corresponds to a specific loss function, such as the cross-entropy loss or the multiclass hinge loss~\citep{crammer2001algorithmic,mohri2018foundations}. 

Another example of VVPs arises in feed-forward neural networks, where a composition of such transformations occurs, each of which corresponds to a layer with a weight matrix $W$ and an activation function $\ell$.
Motivated by this connection, a recent line of work studied VVP in the regime where $\ell$ is Lipschitz continuous and $W$ is constrained within a unit ball, relative to some matrix norm $\|\cdot\|$, centered around an ``initialization'', or reference matrix $W_0$ \citep{daniely2019generalization,DBLP:journals/dg-two-layer-networks,DBLP:journals/corr/vardi-shamir-srebro,magen2023initialization}. 
These studies have yielded a range of sample complexity results depending on the particular choice of a matrix norm and properties of the initialization $W_0$.

In this work, we discuss an arguably more basic and fundamental case of the VVP framework: where $\ell$ is a \emph{convex} (and Lipschitz) loss function and the domain is restricted to a simple unit ball, with respect to the Frobenius norm, centered around a given reference matrix $W_0$. In this scenario, recent work by \cite{magen2023initialization} reveals an interesting finding: while learning is possible within this framework using a specific algorithm (namely stochastic gradient descent, SGD), there exist problem instances where generic empirical risk minimization (ERM) fails. 
As they mention in their work, this finding is analogous to a series of studies within the more general context of Stochastic Convex Optimization (SCO), which established that learnability in SCO is algorithmic dependent and in general, learning through ERM could fail when the problem dimension is sufficiently large~\citep{shalev2010learnability,feldman2016generalization}.

\subsection{Our contributions}

In this work, we present several findings that contribute to a better understanding of the complexity of learning vector-valued predictors (VVPs) with convex loss functions and its connection to stochastic convex optimization (SCO). 
Our main contributions of this paper are summarized as follows:

\begin{enumerate}[label=(\roman*),leftmargin=!]
    \item
    We characterize the exact sample complexity of ERMs within the framework of convex and Lipschitz VVPs, demonstrating a lower bound of $\tOm\br{\ifrac{\rowsnum}{\epsilon^2}}$. Together with a classic result of \cite{maurer2016vector}, this implies that the sample complexity of ERM in the VVP setting is $\tTh\br{\ifrac{\rowsnum}{\epsilon^2}}$.
    In particular, our lower bound provides for an exponential improvement as compared to the lower bound of \cite{magen2023initialization}, that scaled poly-logarithmically with the target dimension $k$, and further includes the tight dependence on $\epsilon$.
    
    \item 
    We present a black-box transformation from general SCO to VVP, that converts a given SCO instance in $d$-dimensions to a convex VVP problem with $\rowsnum=\Theta(d)$ outputs. We show that using any algorithm for the VVP setting to solve the converted problem instance to within $\epsilon$ excess risk using a training sample of size $n$, we can directly recover a solution to the original SCO problem with excess risk $O(\epsilon + \ifrac{1}{\sqrt{n}})$.
\end{enumerate}

Put together, the two results indicate that, in terms of its complexity, VVP bridges between two extreme models: generalized linear models, namely the case $\rowsnum=1$, and general $\convdim$-dimensional SCO, that roughly correspond to $\rowsnum=\Theta(\convdim)$. 
First, our sample complexity bounds for ERM can be seen as interpolating between the classical $\Theta(\ifrac{1}{\epsilon^2})$ sample complexity rate of generalized linear models~\citep{bartlett2002rademacher} and the analogous bound in $d$-dimensional SCO, which is linear in $d$~\citep{feldman2016generalization,carmon2023sample}.
Second, from a more structural perspective, our transformation from SCO to VVP suggests that in the extreme where $k=\Theta(d)$, vector-valued prediction becomes rich enough so as to encompass generic SCO problems.

The revealed connection between linear models and SCO through the lens of VVP is perhaps somewhat surprising, since the two are extensively-studied problems that traditionally differ from one another in terms of techniques and results. Further, it partially addresses a common conceptual criticism of SCO as a learning framework: in SCO, there is no apparent concept of ``prediction'' and losses are rather implicitly assigned to model parameters, whereas VVP is naturally a supervised learning model that explicitly defines a rule $x \rightarrow Wx$ from which predictions are generated and losses are induced.

\subsection{Additional related work}
\label{sec:related-works}

\paragraph{Upper bounds for the sample complexity of VVPs.}
As alluded to in the introduction, closely related to our setting is the work of \cite{maurer2016vector} that gives an upper bound that scales like $O(\rowsnum)$, for convex and Lipschitz predictors with bounded Frobenius norm.
In another work, \cite{DBLP:journals/dg-two-layer-networks} achieved an upper bound that is similar to \cite{maurer2016vector} for non-convex predictors and with bounded difference from a reference matrix $W_0$.
For the specific case of $W_0=0$, \cite{DBLP:journals/corr/vardi-shamir-srebro} shows an upper bound that scales like $O(\log \rowsnum)$ and \cite{magen2023initialization} proved an upper bound independent on $\rowsnum$.
The works of \cite{lei2019data,zhang24multi} studied VVPs with arbitrary initialization under an $\ell_\infty$-Lipschitz condition, a stronger assumption compared to our setting, and showed upper bounds that are logarithmic in $\rowsnum$.
In addition, \cite{zhang24multi} state a similar upper bound in the $\ell_2$-Lipschitz case. In a personal communication with the authors of this paper, it was pointed out that the improved bound, which will appear in an upcoming journal version of their work, in fact holds under an additional smoothness assumption. This is compatible with our results, which indicate that some additional assumptions (though not necessarily smoothness) and potentially new techniques are indeed required to obtain such an improved dependence on $k$ in the standard $\ell_2$-Lipschitz case. Whether smoothness is particularly necessary for obtaining improved bounds in the $\ell_2$- Lipchitz case remains an interesting open problem for further investigation.

\paragraph{Lower bounds for the sample complexity of VVPs.}
For one-hidden-layer neural networks, which is a special case of valued predictors with non-convex loss, \cite{daniely2019generalization} provide a fat-shattering lower bound for the case where $\rowsnum=\colsnum$, crucially rely on the inputs have norm that scales with $\rowsnum$.
Then, \cite{DBLP:journals/corr/vardi-shamir-srebro} referred to the case $\rowsnum\neq \colsnum$ and showed a lower bound of $\Omega(\ifrac{\rowsnum}{\epsilon^2})$ for the sample complexity of a class of non-convex predictors, where the initialization matrix is $W_0=0$ and the $\ell_2$ norm of the prediction matrix is bounded by a constant. Then, \cite{DBLP:journals/dg-two-layer-networks} refer to the class of predictors with bounded Frobenius norm and showed that this class can shatter a training set of $\Omega(\rowsnum)$ examples, assuming that the inputs have norm $\sqrt{\colsnum}$ and that $\rowsnum=O(2^\colsnum)$.
In a recent work, \cite{magen2023initialization} generalized this bound to the case where the Frobenius norm of the distance from an arbitrary initialization matrix is bounded.
In the same work, they discussed convex vector valued predictors and showed a lower bound of $\Omega(\log \rowsnum)$ for the sample complexity of convex predictors.
In this work, we improve their lower bound for convex predictors as we achieve an exponential increase in the dependence in $\rowsnum$. 

\paragraph{Generalized Linear Models.}
In the landscape of learning theory literature, the Generalized Linear Models (GLM) framework stands as one of the most basic and extensively explored settings \citep[e.g.,][]{DBLP:books/daglib/shalev-shwartz2014}, as it captures some fundamental problems like logistic regression and support-vector machines.
In this setting, due to dimension-independent uniform convergence, it is guaranteed that constrained ERM learns with optimal sample complexity of $O(\ifrac{1}{\epsilon^2})$ examples \citep{bartlett2002rademacher}.
A more recent work by \cite{amir2022thinking} show that unregularized gradient methods, such as full-batch Gradient Descent achieve the same sample sample complexity when learning GLMs.

\paragraph{Stochastic Convex optimization.}
Stochastic convex optimization (SCO) is a fundamental theoretical framework widely used for studying common optimization algorithms. This is often justified by the simplicity of the framework and the possibility of a rigorous analysis that can hint at the pros and cons of various optimization techniques in practical setups arising in machine learning.  
The works of \cite{shalev2010learnability,feldman2016generalization,carmon2023sample} demonstrated that, although learnability in this setting is possible (e.g., by Stochastic Gradient Descent) ERM may not learn in this setting since uniform convergence does not generally hold.
Specifically, \citet{carmon2023sample} recently established that the sample complexity of ERM in $d$-dimensional SCO is $\Theta\br{ \ifrac{d}{\epsilon}+\ifrac{1}{\epsilon^2} }$. We note that since our lower bound requires that the number of columns of the vector-valued predictor matrix is $\colsnum=\Theta(n)$ and the total number of parameters is $\Omega(\colsnum n)$, this lower bound does not contradict their upper bound.
Beyond generic ERM, several specific and natural ERM algorithms, which are also frequently used in practice, such as full batch Gradient Descent have been shown to fail in learning this setting~\citep{amir2021sgd,schliserman2024dimension,livni2024sample}.

\section{Problem Setup and Basic Definitions}
\paragraph{Notations.}
For every vector $x\in \R^d$, we denote its $i$th entry by $\vecentry{x}{i}$ and the vector in $\R^{j-i+1}$ which is achieved by taking the entries with index $i\leq k\leq j$ 
by $\vecentries{x}{i}{j}$.
For every $n\in \N$, we denote $[n]=\{1,\ldots,n\}$.
We denote the Frobenius norm of a matrix $M$ by $\norm{M}_F = (\sum_{i,j} M_{i,j}^2)^{1/2}$, and denote a unit ball with respect to $\norm{\cdot}_F$ centered at $W_0$ by $\unitballrowscols$.
Moreover, for every dimension $d$, we denote the $d$-dimensional unit ball around the origin by $\unitballconv$, the $d$-dimensional standard basis by $\{e_1,\ldots,e_d\}$.

\paragraph{Vector-valued prediction.}
Our main setting of interest in this paper is vector-valued prediction with a convex and Lipschitz loss function. 
Let $\mathcal{D}$ be a distribution supported over vectors $x\in \R^\colsnum$ such that $\|x\| \leq 1$. Given a convex and $G$-Lipschitz loss function $\ell$ defined over the $\rowsnum$-dimensional unit ball $\unitballrows \subseteq \mathbb{R}^\rowsnum$, and an initialization matrix $\initmatrix$, the objective is to find a matrix $W \in \unitballrowscols$ with low population loss, defined as the expected value of the loss function over the distribution $\mathcal{D}$, namely
\begin{align*} 
    L(W) = \E_{x \sim \mathcal{D}}[\ell(Wx)].
\end{align*}
To find such a model $W$, the learner uses a set of $\numdata$ training examples $S = \{z_1, \ldots, z_\numdata\}$, drawn i.i.d.\ from the unknown distribution $\mathcal{D}$. Given the sample $S$, the corresponding \emph{empirical loss} (or \emph{risk}) of $W$, denoted $\widehat{L}(W)$, is defined as its average loss over samples in $S$:
\[
    \widehat{L}(W) = \frac{1}{\numdata}\sum_{i=1}^\numdata L(Wx_i).
\]
A population minimizer in this context is any $\optmat$ that minimizes the population risk, namely such that $\optmat \in \arg\min_{W\in \unitballrowscols} L(W)$,
and an empirical risk minimizer (ERM) is any $\ermmat$ that minimizes the empirical risk, namely such that $\ermmat \in \arg\min_{W\in \unitballrowscols} \widehat{L}(W)$.

\paragraph{Stochastic Convex Optimization.}

Another learning model we discuss is Stochastic Convex Optimization (SCO), which is a more general framework that includes (convex) vector-valued prediction as a special case.
In this problem, there is a population distribution $\mathcal{D}$ over an arbitrary instance set $Z$ and a loss function $f: \mathcal{W} \times Z \rightarrow \R$ which is convex and $G$-Lipschitz (for some $G>0$) with respect to its first argument over a domain $\mathcal{W}$.  
For simplicity, we fix in this paper the domain $\mathcal{W}$ to be the $\convdim$-dimensional unit ball around the origin, denoted $\unitballconv$.
Analogously to vector-valued prediction, the population loss with respect to $f$, denoted by $F$, is defined as, 
\begin{align*} 
    F(w) = \E_{z\sim \mathcal{D}}[f(w,z)].
\end{align*}
and the empirical loss, denoted by $\hat F$, is defined as, 
\[
    \widehat{F}(w) = \frac{1}{\numdata}\sum_{i=1}^\numdata f(w,z_i).
\]
and corresponding minimizers as
$\opt$ and  $\erm$, respectively.
\section{Sample Complexity of ERM in Convex Vector-Valued Prediction}
\label{sec_erm}

In this section, we present a tight characterization of the sample complexity of ERM in the setting of convex vector-valued prediction. This result is stated as follows.

\begin{theorem} \label{main_thm_erm_lower}
    Let $\rowsnum$,$\numdata\in \N$. 
    There exist $\colsnum=\Theta(\numdata)$, a reference matrix $W_0\in \R^{\rowsnum\times \colsnum}$, a convex and $1$-Lipschitz loss function $\ell\in \R^\rowsnum\to \R$ and a distribution $\D$ such that in the VVP parameterized by $W_0,\D$ and $\ell$, with constant probability over the choice of the training set $S\sim \D^n$, there exists an ERM $\ermmat$ with 
    $
        L(\ermmat)-L(\optmat)
        =
        \tOm\big( \sqrt{\rowsnum/\numdata} \big)
        .
    $
\end{theorem}
We further show that this lower bound is tight up to logarithmic factors, using a vector-contraction inequality for Rademacher complexity due to \citet{maurer2016vector}, that implies an $O(k/\epsilon^2)$ sample complexity upper bound.  We defer details on this standard derivation to \cref{sec_erm_proofs}, and below focus on proving the lower bound in \cref{main_thm_erm_lower}.

For the case of $\rowsnum\leq O(\log \numdata)$, \cref{main_thm_erm_lower} follows from the well-known lower bound of $\Omega(1/\sqrt{n})$ for learning scalar-valued predictors with convex and $O(1)$-Lipschitz losses (for completeness, we provide a proof in \cref{sec_erm_proofs}).  Thus, henceforth we focus on the case where $\rowsnum\geq \Omega(\log \numdata)$.
Our proof approach in this case is to show that, for large enough column dimension $\colsnum$ and for a certain reference matrix $W_0$, the class of predictors parameterized by matrices in a unit Frobenius-norm ball centered at $W_0$ can shatter 
$\tOm\br{\rowsnum/\epsilon^2}$ examples with margin $\epsilon$.\footnote{A class of functions $\F$ on an input domain $\mathcal{X}$ \emph{shatters} $m$ points $x_1,...,x_m \in \mathcal{X}$ with margin $\epsilon$, if for all $y \in \{0,1\}^m$ we can find $f \in \F$ such that
for all $i \in [m]$, it holds $f(x_i) \leq -\epsilon \ \ \text{if} \ \ y_i = 0 \ \ \text{and} \ \ f(x_i)\geq \epsilon \ \ \text{if} \ \ y_i=1$.}
This is formalized in the following lemma.
\begin{lemma}
\label{lower_bound_shattering}
Let $10000\leq \rowsnum\in \N$ and $\frac{1}{12}\geq\epsilon\geq \sqrt{\rowsnum}2^{-\ifrac{\rowsnum}{312}}$.
There exists column dimension $\colsnum_0=\Theta(\rowsnum/\epsilon^2)$, such that for any $\colsnum\geq \colsnum_0$, there exist a  matrix $W_0\in \R^{\rowsnum \times \colsnum}$ and a loss function $\ell:\R^\rowsnum \to \R$, such that the class of vector-valued predictors %
\[
    \predclass
    := 
    \left\{x\rightarrow \ell(Wx) \;:\; W \in \R^{\rowsnum\times \colsnum}, \frobnorm{W-W_0}\leq 1\right\}
\]
can shatter $\Omega(\rowsnum/\epsilon^2)$ examples with margin $\epsilon$.
\end{lemma}

\cref{lower_bound_shattering} implies \cref{main_thm_erm_lower} via standard arguments; we defer this proof to \cref{sec_erm_proofs} and below focus on proving the lemma, which forms our main contribution in this section.

Before we formally prove \cref{lower_bound_shattering}, let us first outline the main steps and challenges in constructing the lower bound instance.
Our general approach is analogous to the arguments of \cite{magen2023initialization}. They show that for every $n\in \N$ there exists a data set $\{x_1,\ldots,x_n\}$ and labeling $y\in \R^n$, there exists a matrix $W_y$
with $\rowsnum=\Omega(2^n)$ such that for every $i$, $\ell(W_yx_i)=\epsilon$ if $y_i=1$ and $\ell(W_yx_i)=-\epsilon$ if $y_i=0$.
Their approach is to use the exponentially-sized set $\{e_{i}\}_{i=1}^{2^{n}}$ of standard basis vectors and associate every possible labeling $y\in \{0,1\}^{n}$ with a vector $e_y$ in this set and a matrix $W_y$ with $n+1$ columns which its first $n$ columns are the identity matrix and its last column is $e_y$. 
Then, they used a convex loss function $\ell$ constructed such that the predictor $\hat{y}_i=W_yx_i$ output the prediction according to the corresponding label of $x_i=e_i$. 

Our main challenge, however, is to shatter a training set $\{x_1,\ldots,x_n\}$ using matrices $\{W_y\}_{y\in 2^n}$ with only $\rowsnum=\Theta(n)$ rows rather than $\rowsnum=O(2^n)$ rows, as in the construction above.
For this, we employ a construction of a set $U$ of approximately orthogonal vectors in $\R^{O(\rowsnum)}$ with size which is exponential in $\rowsnum$, adapted from \citet{feldman2016generalization}. 
(The specific construction appears in \cref{lem:set_direc_exists} in \cref{sec_erm_proofs}.)
In our construction of hard instance, for we use this construction twice. First, as the columns of the initialization matrix (replacing of the standard basis vectors in \cite{magen2023initialization}) and second, by identifying every possible labeling $y$ with a vector $u_y$ in this set (instead of $e_y$ in \cite{magen2023initialization}) and using the matrices $W_y$ which their last columns are $u_y$.

Finally, for getting the correct dependency with respect to $\epsilon$, we add $\Theta(1/\epsilon^2)$  columns to the prediction matrix, and modify the previous construction such that every possible labeling $y$ is identified not just with a single vector $u_y$, but rather with a sequence of vectors in the set $U$ alluded to in \cref{lem:set_direc_exists}. This adding of the matrix enables the class $\predclass$ to shatter a larger amount of possible labelings.

We conclude this section with the proof of \cref{lower_bound_shattering}.
\begin{proof}[of \cref{lower_bound_shattering}]
We denote $J=\ifrac{12}{\epsilon^2}$. 
We use the set $U:=U_{\ifrac{\rowsnum}{13}}$ of $(\ifrac{\rowsnum}{13})$-dimensional 
nearly-orthogonal vectors,
given in \cref{lem:set_direc_exists} with $|U|=2^{\ifrac{\rowsnum}{156}}$, and use an arbitrary enumeration of this set $U=\{u_1,\ldots, u_{|U|}\}$. 
For every $u\in U \subseteq \R^{\frac{\rowsnum}{13}}$, we define %
the vector $u'\in \R^{\rowsnum}$ 
which is for $1\leq i\leq \frac{\rowsnum}{13}$, $\vecentry{\Tilde{u}}{i}=\vecentry{u}{i}$ and other entries equal zero.
Let $\colsnum=(\frac{\rowsnum}{13}+\frac{1}{12})J$ and $W_0\in \R^{\rowsnum\times \colsnum}$ be the following matrix (note that by the lower bound for $\epsilon$, it holds that $\frac{\rowsnum J}{13}\leq 2^{\frac{\rowsnum}{156}}$):
\[
    W_0 =
    \epsilon ~ \begin{pNiceArray}{c|c|c|c|c}[margin=5pt,cell-space-limits=4pt]
         u_1  & u_2 & \ldots & u_{\frac{mJ}{7}}  & \mathbf{0}
        \\
        \hline
        \mathbf{0}&\mathbf{0}&\mathbf{0}&\mathbf{0} & \mathbf{0}
    \end{pNiceArray}.
\]
Now, for every $i\in[\frac{12\rowsnum}{13}]$ and $j\in\left[\frac{J}{12}\right]$ we define $x_{i,j}=\frac{1}{\sqrt{2}}e_{\frac{12\rowsnum}{13}(j-1)+i}+ \frac{1}{\sqrt{2}}e_{\frac{\rowsnum J}{13}+j}$.
We show that the set $S=\{x_{i,j}: i\in[\frac{12 \rowsnum}{13}
],j\in\left[\frac{J}{12}\right]\}$ can be shattered by $\predclass$.

For this, we use the set $\hat{U}:=U_{\frac{12\rowsnum}{13}}$, of $\frac{12\rowsnum}{13}$-dimensional 
nearly-orthogonal vectors, given in \cref{lem:set_direc_exists} with $|\hat{U}|=2^{\frac{\rowsnum}{13}}$ and use an arbitrary enumeration of this set $\hat{U}=\{\hat{u}_1,\ldots \hat{u}_{|\hat{U}|}\}$. 
For the rest of the proof, we refer to every vector $z\in\{0,1\}^{\frac{\rowsnum J}{13}}$ as a sequence of $\frac{J}{12}$ vectors in $\{0,1\}^{\frac{12\rowsnum}{13}}$, $\vecpart{z}{1},\ldots,\vecpart{z}{\frac{J}{12}}$, where for every $r\in \frac{J}{12}$, 
$\vecpart{z}{r}=\vecentries{z}{(r-1)(\frac{12\rowsnum}{13})+1}{\frac{12\rowsnum r}{13}}$.
Now, since we refer to every possible labeling for S, $y\in \{0,1\}^{\frac{\rowsnum J}{13}}$ as a sequence of $\frac{J}{12}$ vectors, $\vecpart{y}{1},\ldots,\vecpart{y}{\frac{J}{12}}$, it is possible to identify such labeling $y$ with a sequence $(\hat{u}_{\vecpart{y}{1}},\hat{u}_{\vecpart{y}{2}}\ldots \hat{u}_{\vecpart{y}{\frac{J}{12}}})\in \hat{U}^\frac{J}{12}$.
For every such $y$, we define the following matrix: 
\[
    W_y = \epsilon ~
    \begin{pNiceArray}{c|c|c|c|c}[margin=5pt,cell-space-limits=4pt]
        \mathbf{0} & \mathbf{0}
        \\
        \hline
        \mathbf{0} & \hat{u}_{\vecpart{y}{1}} & \hat{u}_{\vecpart{y}{2}} & \ldots & \hat{u}_{\vecpart{y}{\frac{J}{12}}}
    \end{pNiceArray}.
\]
By the definition of $J$, it holds that $\frobnorm{W_y}\leq 1$.
Now,
for every $\hat{u}\in \hat{U}$, we define the vector $\Tilde{u}\in \R^{\rowsnum}$ which is zero in its first $\frac{\rowsnum}{13}$ entries and 
for the rest of the entries, $\frac{\rowsnum}{13}+1\leq i\leq \rowsnum$, it holds that, $\vecentry{\Tilde{u}}{i}=\vecentry{\hat{u}}{i-\frac{\rowsnum}{13}}$.
We turn to define the loss function $\ell:\R^\rowsnum\to \R$. For this, we define the following set
\[A=\set*{z\in \{0,1\}^\frac{12\rowsnum}{13}, \hat{z}\in \{0,1\}^{\frac{\rowsnum J}{13}}, r\in[\frac{12 \rowsnum}{13}], p\in[\frac{J}{12}] \quad\Big|\quad \forall 1\leq b\leq \frac{J}{12} \;:\; \vecpart{\hat{z}}{b}=z, \vecpart{\hat{z}}{p}(r)=\epsilon}.\]
The the loss function $\ell$ is defined as following,
\begin{equation*}
\ell (\hat{y})=2\sqrt{8}\max_{(\hat z, z, r, p)\in A} \left\{\frac{3}{\sqrt{8}}\epsilon, 
\max\{\frac{\epsilon
}{\sqrt{8}}, \Tilde{u}_z^T\hat{y}\} + \max\{\frac{\epsilon}{\sqrt{8}}, u_{r+\frac{12\rowsnum}{13}(p-1)}^{'T}\hat{y}\}\right\} - 7\epsilon.
\end{equation*}
The function is $1$-Lipschitz and convex as a maximum over linear $1$-Lipschitz functions.
For every $y\in \{0,1\}^\frac{\rowsnum J}{13}$ we define $W'_{y}=W_0+W_y$.
Let $x_{i,j}\in S$. 
Then,
\[W'_{y}x_{i,j}= \frac{1}{\sqrt{2}}\epsilon u'_{\frac{12 \rowsnum}{13}(j-1)+i} + \frac{1}{\sqrt{2}}\epsilon \Tilde{u}_{\vecpart{y}{j}},\]
and
\begin{align*}
&\mqquad
\ell(W'_{y}x_{i,j}) =
\\
&=
2\sqrt8\max_{(\hat z, z, r, p)\in A} \left\{\frac{3}{\sqrt{8}}\epsilon, 
\max\{\frac{\epsilon
}{\sqrt{8}}, \Tilde{u}_z^TW'_{y}x_{i,j}\} + \max\{\frac{\epsilon}{\sqrt{8}}, u_{r+\frac{12\rowsnum}{13}(p-1)}^{'T}W'_{y}x_{i,j}\}\right\} -7\epsilon
\\
&=2\sqrt8\epsilon\max_{(\hat z, z, r, p)\in A} \left\{\frac{3}{\sqrt{8}}, 
\max\{\frac{1
}{\sqrt{8}}, \frac{1}{\sqrt{2}}   \hat{u}_z^T\hat{u}_{\vecpart{y}{j}}\} + \max\{\frac{1}{\sqrt{8}}, \frac{1}{\sqrt{2}}  u_{r+\frac{12\rowsnum}{13}(p-1)}^{T}  u_{\frac{12\rowsnum}{13}(j-1)+i}\}\right\} - 7\epsilon
\\&=2\epsilon\max_{(\hat z, z, r, p)\in A} \left\{3, 
\max\{1, 2 \hat{u}_z^T\hat{u}_{\vecpart{y}{j}}\} + \max\{1, 2 u_{r+\frac{12\rowsnum}{13}(p-1)}^{T}  u_{\frac{12\rowsnum}{13}(j-1)+i}\}\right\}-7\epsilon.
\end{align*}

If $y_{i,j}=\vecentry{\vecpart{y}{j}}{i}=1$, the maximum of the first term is attained at $z=\vecpart{y}{j}$ and the maximum of the sum of the terms is attained at $\hat{z}$ such that for every $b$, $\vecpart{\hat{z}}{b}=\vecpart{y}{j}$. Moreover, since $\vecentry{\vecpart{y}{j}}{i}=1$, it holds that $\vecentry{\vecpart{\hat{z}}{b}}{i}=1$ for every $b$, and particularly, for $p=j$, $r=i$, $\vecentry{\vecpart{\hat{z}}{p}}{r}=1$,  thus, since $p=j,r=i$ gives the maximal inner product and the condition of the max holds, the maximum of the second term is attained at $p=j,r=i$, and,
\begin{align*}
    \ell(W'_{y}x_{i,j})
&=2\epsilon\max_{(\hat z, z, r, p)\in A} \left\{3, 
\max\{1, 2\max \hat{u}_z^T\hat{u}_{\vecpart{y}{j}}\} + \max\{1, 2\max u_{r+\frac{12\rowsnum}{13}(p-1)}^{T}  u_{\frac{12\rowsnum}{13}(j-1)+i}\}\right\}-7\epsilon
\\&=2\epsilon\max_{(\hat z, z, r, p)\in A} \left\{3, 
2\hat{u}_{\vecpart{y}{j}}^T\hat{u}_{\vecpart{y}{j}} + 2u_{i+\frac{12\rowsnum}{13}(j-1)}^{T}  u_{\frac{12\rowsnum}{13}(j-1)+i}\right\}-7\epsilon
\\&=8\epsilon- 7\epsilon
\\&=\epsilon.
\end{align*}
If $y_{i,j}=\vecentry{\vecpart{y}{j}}{i}=0$, and there exists $r$ such that $\vecentry{\vecpart{y}{j}}{r}=1$, the maximum of the first term is attained at $z=\vecpart{y}{j}$ and the maximum of the sum of the terms is attained at $\hat{z}$ such that for every $b$, $\vecpart{\hat{z}}{b}=\vecpart{y}{j}$. Moreover, since $\vecentry{\vecpart{y}{j}}{i}=0$, it holds that $\vecentry{\vecpart{\hat{z}}{b}}{i}=0$ for every $b$, thus, for every $r,p$ such that $\vecentry{\vecpart{\hat{z}}{p}}{r}=1$, it holds that $r+\frac{12\rowsnum}{13}(p-1)\neq i+\frac{12\rowsnum}{13}(j-1)$ and $u_{r+\frac{12\rowsnum}{13}(p-1)}^{T}  u_{\frac{12\rowsnum}{13}(j-1)+i}\leq \frac{1}{2}$. Then,
\begin{align*}
    \ell(W'_{y}x_{i,j})&
=2\epsilon\max_{(\hat z, z, r, p)\in A} \left\{3, 
\max\{1, 2\max \hat{u}_z^T\hat{u}_{\vecpart{y}{j}}\} + \max\{1, 2\max u_{r+\frac{12\rowsnum}{13}(p-1)}^{T}  u_{\frac{12\rowsnum}{13}(j-1)+i}\}\right\}- 7\epsilon
\\&=2\epsilon\max_{(\hat z, z, r, p)\in A}\left\{3, 
2\hat{u}_{\vecpart{y}{j}}^T\hat{u}_{\vecpart{y}{j}} + 1\right\}- 7\epsilon
\\&=6\epsilon- 7\epsilon
\\&=-\epsilon.
\end{align*}
If $y_{r,j}=\vecentry{\vecpart{y}{j}}{r}=0$ for every $r$, for every $\hat{z}$ such that for every $b$, $\vecpart{\hat{z}}{b}=\vecpart{y}{j}$ it holds that $\hat{z}=\{0\}^{\frac{\rowsnum J}{13}}$. Then, the set where the maximum is applied is empty and
\begin{align*}
\ell(W'_{y}x_{i,j})=2\epsilon \cdot 3 - 7\epsilon= -\epsilon.
\end{align*}
We showed that $S$ can be shattered by $\predclass$, which implies the lemma.
\end{proof}

\section{Black-box Transformation from SCO to VVP}
\label{sec_reduct}
In this section we provide our second main result which constitutes a black-box conversion between SCO and learning of vector-valued predictors. 
Namely, we show that there exists a initialization matrix $\initmatrix\in \R^{\rowsnum\times\colsnum}$, such that any $\convdim$-dimensional stochastic optimization problem, with loss function $f$ and distribution $\D$, can be converted to a vector-valued predictions problem over $\unitballrowscols$ with $\rowsnum=O(\convdim)$.

\subsection{The transformation}

Let us first outline our transformation.
Consider any SCO instance in $d$-dimensions characterized by a distribution $\D$ over sample space $Z$, and a convex and $1$-Lipschitz loss function $f:\unitballconv\times Z\to \R$,
and consider an algorithm $\mathcal{A}$ with a guarantee that for every VVP problem, using any training set $S'$ with $n$ examples that are sampled i.i.d from the corresponding distribution, denoted as $\D'$,
outputs a model $W(S')$ which has $\epsilon(n)$-sub optimal population loss.
The conversion uses a training set $S=\{z_1\ldots z_{2n}\}$ of $2n$ examples sampled i.i.d.\ from $\D$, and takes the following form: 
\begin{enumerate}[label=(\roman*),leftmargin=!]
    \item Construct a VVP instance $\mathcal{P}$ as follows:
    \begin{itemize}[leftmargin=!]
    \item The dimensions of the VVP problem are $\colsnum=2n+1$, $\rowsnum=\convdim+2$.
    \item The reference matrix $W_0 \in \R^{\rowsnum\times\colsnum}$ is as follows, 
    \[
        W_0 = c ~ \begin{pNiceArray}{c|c|c|c|c}[margin=5pt,cell-space-limits=4pt]
      \phi(1)  & \phi(2) & \ldots & \phi(2n)  & \mathbf{0}
      \\
      \hline
      \mathbf{0}&\mathbf{0}&\cdots&\mathbf{0} & \mathbf{0}
    \end{pNiceArray}.
    \]
    where $c>0$ is a parameter and $\phi$ is a mapping $\phi:[2n]\to \R^2$ defined below.
        \item The distribution $\D'$ is the uniform distribution on $\{x_1,\ldots,x_{2n}\}$ where $x_i=e_i+e_{2n+1}$ for all $i$.
        
        \item The loss function $\ell : \R^k \to \R$ is defined as
        \begin{equation}
            \label{loss_reduction}
            \ell(\hat{y})=\max_{j\in[2n]} \cbr[\Big]{ \big\langle \vecentries{\hat y}{1}{2},\phi(j) \big\rangle + f(\vecentries{\hat y}{3}{\rowsnum},z_j) },
        \end{equation}
    \end{itemize}

\item Sample a training set $S'$ with $n$ examples drawn i.i.d.~from $\D'$, and use $\mathcal{A}$ with $S'$ to solve $\mathcal{P}$ and obtain a solution matrix $W(S') \in \unitballrowscols$. 
\item Return the vector $w(S')$ formed by the $\convdim$ last entries of the $(2n+1)$th column of $W(S')$.
\end{enumerate}

Here, the mapping $\phi$ is an embedding of the integers $1,\ldots,2n$ into the unit sphere in two dimensions, via 
$\phi(j)=\left(\sin\left(\ifrac{\pi j}{4n}\right), \cos\left(\ifrac{\pi j}{4n}\right)\right)^T$. 
Note that, since the loss function $\ell$ defined in \cref{loss_reduction} is convex and 2-Lipschitz (as the maximum of convex functions is also a convex function), $\mathcal{P}$ is a valid VVP problem which $\mathcal{A}$ can be used to learn.

We show that when running the algorithm $\mathcal{A}$ on $\mathcal{P}$, the solution $w(S')$ emitted by the conversion satisfies the following theorem.
\begin{theorem}\label{scotopredreduction}
Consider any SCO instance in $d$-dimensions characterized by a distribution $\D$ over sample space $Z$, and a convex and $1$-Lipschitz loss function $f:\unitballconv\times Z\to \R$ that further satisfies $|f(w,z)|\leq b$ for every $w,z$.
Let $\mathcal{P}$ be the corresponding VVP problem 
as defined by the conversion above for $\delta>0$ given by \cref{set_direction_exists_R_2} and $c=\ifrac{4b}{\delta}$.
Let $\mathcal{A}$ be an algorithm with a guarantee that for every VVP problem, using any training set $S'$ with $n$ examples that are sampled i.i.d from the corresponding distribution,
outputs a model $W(S')$ with $$\E\sbr{ L(W(S')) -  L(\optmat) } \leq \epsilon(n).$$
Then, when running the algorithm $\mathcal{A}$ on $\mathcal{P}$, the solution $w(S')$ emitted by the conversion
satisfies,
$$\E\sbr{ F(w(S'))-F(w^*) } \leq 2\epsilon(n)+ \frac{10}{\sqrt {n}}.$$
\end{theorem}

Before we prove \cref{scotopredreduction}, let us review the main ideas that we used for constructing this conversion.

First, we aim to represent an arbitrary unknown distribution $\mathcal{D}$ using a distribution $\mathcal{D}'$ over the unit ball $\unitballcols$ and a finite set of samples $S=\{z_1,\ldots,z_n\}$ sampled i.i.d.\ from $\mathcal{D}$. To achieve this, we model not $\mathcal{D}$ directly, but rather its empirical distribution, denoted as $\hat{\mathcal{D}}$, which is the uniform distribution over $S$ and, when taking expectation over $S$, approximates $\mathcal{D}$. To implement this, we associate each example $z_i$ with a standard basis vector $e_i$, and define the distribution $\mathcal{D}'$ as the uniform distribution over the set $\{e_1,\ldots,e_n\}$.

Second, we show how to utilize a one-parameter loss function $\ell:\mathbb{R}^\rowsnum\to \mathbb{R}$ to model the two-parameter loss function $f:\mathbb{R}^\convdim\times Z \to \mathbb{R}$, where $f$ receive a model $w$ and an example $z$ as an input. Our aim is, given $w\in \R^d$ which is a proposed solution for the SCO problem, to construct a function $\ell:\R^\rowsnum\to \R$ and a matrix $W\in \R^{\rowsnum\times \colsnum}$ such that $\rowsnum=O(\convdim)$ and \begin{equation}
\label{condtion_reduction}
    \forall i\in [n] ~:\quad
    \ell(Wx_i)\approx f(w,z_i)
    .
\end{equation} To achieve this, we use the embedding $\phi$ that was defined above.
This embedding satisfies the following lemma,
\begin{lemma} 
\label{set_direction_exists_R_2}
    Let $a\geq 2$. Let $\phi:[a]\to \R^2$ be the embedding such that for every $j\in[a]$, $\phi(j)=\left(\sin\left(\ifrac{\pi j}{2a}\right),\cos\left(\ifrac{\pi j}{2a}\right)\right)^T$. Then, $\|\phi(i)\|=1$ and
    there exist $\delta>0$ such that for every $i\neq j \in [a]$, it holds that $\langle \phi(i),\phi(j) \rangle \leq 1-\delta$.
\end{lemma}

Specifically, using $\delta$ and $\phi$ defined in \cref{set_direction_exists_R_2} for $a=n$,
we set $\rowsnum=\convdim+2$ and utilize the first two entries of $\hat{y}_i:=Wx_i\in \mathbb{R}^{\convdim+2}$ to encode the corresponding index $i$ using $\phi(i)$.
Then, by incorporating a $\max$ term over all $i\in[n]$ into $\ell$ and set $c=\ifrac{4b}{\delta}$ (where $b$ is a bound on the values of $f$), this index can be decoded and the corresponding loss function $f(\cdot,z_i)$ can be applied. 
Finally, for constructing the matrix $W$ we add another column with index $n+1$ to the matrix and modify $\mathcal{D}'$ to represent the uniform distribution over $\{e_i+e_{n+1}\}_{i=1}^n$. 
This change of the distribution makes the last $\convdim$ entries of $\hat{y}_i$ equal to the last $\convdim$ entries of the added column (for every $i$). 
Then, when the latter is used as placeholder for $w$, we get that \cref{condtion_reduction} holds.

Third, we relate the population loss of the two problems. For this, we employ the technique of double sampling. Specifically, we use a set $S=\{z_1,\ldots,z_{2n}\}$ sampled i.i.d.\ from $\D$ and our conversion samples only $n$ examples from $\mathcal{D}'$. This ensures that at least half of the examples will not appear in the training set of the prediction problem. By \cref{condtion_reduction}, when taking the expectation over $S$, the expected prediction loss on such samples will be equal to the population loss in the convex optimization problem. Finally, by bounding the loss $\ell(\hat y_i)$ for examples that appears in the training set of the prediction problem, we get a population guarantee for the SCO problem.

Now, we turn to prove \cref{scotopredreduction}.  
\begin{proof}[of \cref{scotopredreduction}]
First, defining $W:=W(S')$ and denoting the columns of any matrix $M\in \R^{\rowsnum\times \colsnum}$ as $M_1,\ldots M_{\colsnum}$. By the definitions of $w(S'),\ell$ and $W_0$, it holds that
\begin{align*}
    L(W)&=\frac{1}{2n}\sum_{i=1}^{2n} \max_{j\in[2n]} \left(\langle\vecentries{W_i}{1}{2},\phi(j)\rangle+f(w(S'),z_j)\right)
    \\&\geq \frac{1}{2n}\sum_{i=1}^{2n} \langle\vecentries{W_i}{1}{2},\phi(i)\rangle+f(w(S'),z_i).
\end{align*}
Now, we define the following matrix, 
\[
    \Tilde{W} =
    \begin{pNiceArray}{c|c}[margin=4pt,cell-space-limits=4pt]
        W_0 & \mathbf{0}
        \\
        \hline
        \mathbf{0} & w^*
    \end{pNiceArray}.
\]  
By \cref{set_direction_exists_R_2}, for every $i\neq j\in[2n]$,
\begin{align*}
    &\mqquad
    \left(\langle c\phi(i),\phi(i)\rangle+f(w^*,z_i)\right)- \left(\langle c\phi(i),\phi(j)\rangle+f(w^*,z_j)\right)\\&=
   c\langle \phi(i),\phi(i)\rangle-c\langle \phi(i),\phi(j)\rangle+f(w^*,z_i)-f(w^*,z_j)
    \\&\geq c\delta-2b
    \\& > 0.
\end{align*}
As a result, 
\begin{align*}
L(\optmat) &\leq L(\Tilde{W})
\\&=\frac{1}{2n}\sum_{i=1}^{2n} \max_{j\in[2n]} \left(\langle\vecentries{\Tilde{W}_i}{1}{2},\phi(j)\rangle+f(w^*,z_j)\right)
\\&=\frac{1}{2n}\sum_{i=1}^{2n} \max_{j\in[2n]} \left(\langle c\phi(i),\phi(j)\rangle+f(w^*,z_j)\right)
\\&=\frac{1}{2n}\sum_{i=1}^{2n} \langle \vecentries{W_{0_{i}}}{1}{2},\phi(i)\rangle + f(w^*,z_i).
\end{align*}
Now, combining with the Cauchy-Schwartz and Jensen inequalities we get that,
\begin{align*}
L(W) - L(\optmat) &\geq \frac{1}{2n}\sum_{i=1}^{2n} f(w(S'),z_i)-f(\opt,z_i)+ \frac{1}{2n}\sum_{i=1}^{2n} \langle\vecentries{W_i}{1}{2}-\vecentries{W_{0_{i}}}{1}{2},\phi(i)\rangle
\\&\geq \frac{1}{2n}\sum_{i=1}^{2n} f(w(S'),z_i)- f(\opt,z_i)-\frac{1}{2n}\sum_{i=1}^{2n} \|\vecentries{W_i}{1}{2}-\vecentries{W_{0_{i}}}{1}{2}\| 
\\&\geq \frac{1}{2n}\sum_{i=1}^{2n} f(w(S'),z_i)- f(\opt,z_i)-\sqrt{\frac{1}{2n}\sum_{i=1}^{2n} \|\vecentries{W_i}{1}{2}-\vecentries{W_{0_{i}}}{1}{2}\|^2}
\\&\geq \frac{1}{2n}\sum_{i=1}^{2n} f(w(S'),z_i)- f(\opt,z_i)-\frac{1}{\sqrt{2n}},
\end{align*}
where in the last inequality we used the fact that 
 $\frac{1}{2n}\sum_{i=1}^{2n} \|\vecentries{W_i}{1}{2}-\vecentries{W_{0_{i}}}{1}{2}\|^2\leq \frobnorm{W-W_0}^2\leq 1$.
Then, we denote $I=\{i_1,\ldots,i_p\}$ the set of indices $i\in[2n]$ that sampled from $\D'$ as the data set of the VVP problem. Since $p\leq n$, we can add $n-p$ additional items from $[2n]\setminus I$ to $I$ to create a set $\Tilde{I}=\{i_1,\ldots i_n\}$.
Fixing $S'$ and taking expectation on $S=\{z_1,\ldots,z_{2n}\}$ (note that $w(S')$ and samples $z_i$ are independent if $i\notin \Tilde{I}$ ), we get
\begin{align*}
    &\mqquad
    \E_{S'}\left[L(W - L(\optmat)\right] +\frac{1}{\sqrt{2n}}\\&\geq 
    \E_{\{z_i:i\in \Tilde{I}\}} \frac{1}{2n}\sum_{i\in \Tilde{I}}f(w(S'),z_i)- f(\opt,z_i)+\E_{\{z_i:i\in S\}}\frac{1}{2n}\sum_{i\notin I'} f(w(S'),z_i)- f(\opt,z_i)
    \\&\geq 
    \E_{z_{i_1},\ldots,z_{i_n}} \left[\frac{1}{2n}\sum_{j=1}^n f(w(S'),z_{i_{j}})- f(\opt,z_{i_{j}})\right]+\E_{z_{i_1},\ldots,z_{i_p}} \left[\frac{1}{2}F(w(S'))-\frac{1}{2}F(\opt)\right].
\end{align*}
Now, taking expectation over $S'$, we get by Lemma 1 of \cite{benignunderfit} (see \cref{lem:star_aerm} in \cref{sec_erm_proofs}), 
\begin{align*}
   &\mqquad
   \E\left[L(W) - L(\optmat) \right]+\frac{1}{\sqrt{2n}}
    \\&\geq  \E\left[\frac{1}{2n}\sum_{j=1}^n f(w(S'),z_{i_{j}})- f(\opt,z_{i_{j}})\right]+\E \left[\frac{1}{2}F(w(S'))-\frac{1}{2}F(\opt)\right]
    \\&\geq - \frac{2}{\sqrt{n}}\tag{\cref{lem:star_aerm}}+ \frac{1}{2}\E\left[F(w(S'))-\frac{1}{2}F(\opt)\right] .
\end{align*}
The theorem follows by the guarantee on $\mathcal{A}$ and arranging the inequality.
\end{proof}

\subsection*{Acknowledgments}

Funded by the European Union (ERC, OPTGEN, 101078075).
Views and opinions expressed are however those of the author(s) only and do not necessarily reflect those of the European Union or the European Research Council. Neither the European Union nor the granting authority can be held responsible for them.
This work received additional support from the Israel Science Foundation (ISF, grant numbers 2549/19 and 3174/23), from the Len Blavatnik and the Blavatnik Family foundation, and from the Adelis Foundation.
\bibliography{main.bib}
\newpage
\appendix

\section{Proofs of \cref{sec_erm}}
\label{sec_erm_proofs}

\subsection{Proof of Upper Bound}

First, we use the result of \cite{maurer2016vector} to show an upper bound for the sample complexity of vector valued predictors using Rademacher Complexity. The result that we show is,
    \begin{theorem}
\label{main_thm_erm_upper}
    Let $\rowsnum$,$\numdata\in \N$. For every $\colsnum\in N$, $\initmatrix \in \R^{\rowsnum \times \colsnum}$, convex and $G$-Lipschitz loss function $\ell: \unitballrowscols\to \R$, distribution $\D$ over $\unitballcols$, it holds that,
         \[\E_{S\sim \D^\numdata}\brk[s]2{L(\ermmat)-L(\optmat)}
    = O\left(\frac{\sqrt{\rowsnum}}{\sqrt{\numdata}}\right).\]
\end{theorem}
In the proof we use the standard bound of the generalization error, via the Rademacher complexity of the class (e.g. \cite{bartlett2002rademacher}), we have that:
\begin{align*}
    \E_{S\sim \D^n}\brk[s]2{\sup_{W\in\unitballrowscols}{\brk[c]1{L(W)-\hat{L}(W)}}}
    &\leq
    2\E_{S\sim \D^n}\brk[s]{R_S(\ell\circ \unitballrowscols)},
\end{align*}
Where we notate the function class:
\[\ell\circ \unitballrowscols= \{ x\to \ell(Wx): W\in  \unitballrowscols\}.\]
and $R_S(\ell\circ \unitballrowscols)$ is the Rademacher complexity of the class $\ell\circ \unitballrowscols$. Namely:
\begin{equation}\label{eq:rad}
R_S(\ell \circ \unitballrowscols):= \E_{\sigma}\left[\sup_{h\in \ell \circ \unitballrowscols}\frac{1}{n}\sum_{x_i\in S} \sigma_i h(x_i)\right],
\end{equation}
and $\sigma_1,\ldots, \sigma_{n}$ are i.i.d.~Rademacher random variables.
Now we use the contraction lemma for vector valued predictors given in \cite{maurer2016vector}.
\begin{lemma}[\cite{maurer2016vector}, Corollary 4]
\label{lemma Rademacher}Let $\rowsnum\in \N$ and $\mathcal{X}$ be any set, $\left(
x_{1},...,x_{n}\right) \in \mathcal{X}^{n}$, let $\mathcal{F} $ be a class
of functions $f:\mathcal{X}\rightarrow \R^m$ and let $h_{i}:\R^\rowsnum\rightarrow
\mathbb{R}
$ be $G$-Lipschitz functions. Then,%
\begin{equation*}
\mathbb{E}\sup_{f\in \mathcal{\F} }\sum_{i}\sigma _{i}h_{i}\left( f\left(
x_{i}\right) \right) \leq \sqrt{2}G\mathbb{E}\sup_{f\in \mathcal{F}
}\sum_{i,j}\sigma _{ij}f_{j}\left( x_{i}\right) ,
\end{equation*}%
where $\sigma _{ij}$ is an independent doubly indexed Rademacher sequence
and $f_{j}\left( x_{i}\right) $ is the $j$-th component of $f\left(
x_{i}\right) $.\bigskip 
\end{lemma}
We derive the following lemma
\begin{lemma}
\label{uc_upper}
    Let $\rowsnum$,$\numdata\in \N$. For every $\colsnum\in N$, $\initmatrix \in \R^{\rowsnum \times \colsnum}$, convex and $G$-Lipschitz loss function $\ell: \unitballrowscols\to \R$, distribution $\D$ over $\unitballcols$, it holds that,
         \[\E_{S\sim \D^\numdata}\brk[s]2{\sup_{W\in \unitballrowscols}\brk[c]1{L(W)-\hat{L}(W)}}
    \leq \frac{2\sqrt{2}L\sqrt{\rowsnum}}{\sqrt{\numdata}}.\]
\end{lemma}
\begin{proof}
Let $S=\{x_1,\ldots x_n\}$.
    First,  for $\F=\{Ax_i \mid \frobnorm{A}\leq 1\}$, denoting the $j$-th row of any matrix $A$ as $A_j$, and defining
    $h_i(w)=\ell (W_0x_i+w)$, by \cref{lemma Rademacher} it holds that,
    \begin{align*}
   R_S(\ell \circ \unitballrowscols)&= \E_{\sigma}\left[\sup_{W\in \unitballrowscols}\frac{1}{n}\sum_{x_i\in S} \sigma_i \ell(Wx_i)\right]\\&=
    \E_{\sigma}\left[\sup_{W\in \unitballrowscols}\frac{1}{n}\sum_{x_i\in S} \sigma_i \ell(Wx_i-W_0x_i+W_0x_i)\right]\\&=
    \E_{\sigma}\left[\sup_{W\in \unitballrowscols}\frac{1}{n}\sum_{x_i\in S} \sigma_i h_i((W-W_0)x_i)\right]
    \\&=
    \E_{\sigma}\left[\sup_{\frobnorm{A}\leq 1}\frac{1}{n}\sum_{x_i\in S} \sigma_i h_i(Ax_i)\right]
    \\&\leq \frac{\sqrt{2}G}{n}\mathbb{E}_\sigma \left[\sup_{\frobnorm{A}\leq 1}\sum_{i,j}\sigma _{ij}A_{j}^Tx_{i}\right]
    \\&\leq \left[\frac{\sqrt{2}G}{n}\mathbb{E}_\sigma \sup_{\frobnorm{A}\leq 1}\sum_{j}\sum_{i}\sigma _{ij}A_{j}^Tx_{i}\right]
    \end{align*}
    Now, if $D$ is the matrix that its $j$th column is $\sum_{i}\sigma _{ij}A_{j}^Tx_{i}$, we get,
    \begin{align*}
         R_S(\ell \circ \unitballrowscols)&\leq \frac{\sqrt{2}G}{n}\mathbb{E}_\sigma \sup_{\frobnorm{A}\leq 1} Tr(AD)
         \\&\leq \frac{\sqrt{2}G}{n}\mathbb{E}_\sigma \sup_{\frobnorm{A}\leq 1}\frobnorm{A}\E_\sigma\frobnorm{D}
         \\&\leq \frac{\sqrt{2}G}{n}\E_\sigma\frobnorm{D}
         \\&=\frac{\sqrt{2}G}{n}\E_\sigma\sqrt{\sum_{j}\|\sum_{i}\sigma _{ij}x_{i}\|^2}
      \\&\leq\frac{\sqrt{2}G}{n}\sqrt{\sum_{j}\sum_{i}\|x_{i}\|^2}
         \\&\leq\frac{\sqrt{2}G\sqrt{\rowsnum}}{\sqrt n}.
    \end{align*}
    The lemma follows by combining everything together.
\end{proof}
For finalizing the proof of \cref{main_thm_erm_upper} we use the following lemma from \cite{benignunderfit}.
\begin{lemma}[\cite{benignunderfit}, Lemma 1]
\label{lem:star_aerm}
    Let $W \subseteq \R^d$ with diameter $D$, $\Z$ any distribution over $Z$, and  $f: W \times Z \to 
    \R$ convex and $G$-Lipschitz in the first argument. For every sample set $S=\{z_1,\ldots,z_n\}$ sampled i.i.d from $\Z$, let $w^\star_S=\arg\min \hat{F}(w)$ the empirical risk minimizer. Then
\begin{align*}
    \E_S[\hat{F}(w^*) - \hat{F}(w^*_S)] 
    \leq \frac{4 G D}{\sqrt n}
    .
\end{align*}
\end{lemma}
Now, we can derive \cref{main_thm_erm_upper}.
\begin{proof} [of \cref{main_thm_erm_upper}]
    By \cref{lem:star_aerm} and \cref{uc_upper}, we know that
    \begin{align*}
        &\E_{S\sim \D^\numdata}\brk[s]2{L(\ermmat)-L(\optmat)}
    \\&=\E_{S\sim \D^\numdata}\brk[s]2{L(\ermmat)-\hat{L}(\ermmat)}+\E_{S\sim \D^\numdata}\brk[s]2{\hat L(\ermmat)-\hat{L}(\optmat)}
    \\&
    \leq\frac{2\sqrt{2}G\sqrt{\rowsnum}}{\sqrt n} + \frac{4G}{\sqrt{n}}
    \\&\leq \frac{10G\sqrt{\rowsnum}}{\sqrt n}.
    \end{align*}
\end{proof}
\subsection{Proofs of Lower Bound}
First, we prove the following lemma, which implies the lower bound for the case of $\rowsnum= O(\log \numdata)$,
\begin{lemma}
\label{lower_bound_shattering_small_m}
Let $\rowsnum\in \N$ and $0\leq \epsilon\leq \frac{1}{\sqrt{\numdata}}$. Then exists a dimension $\colsnum_0$ and a matrix $W_0$ such that for any $\colsnum\geq \colsnum_0=\Theta\left(\frac{1}{\epsilon^2}\right)$, $\predclass$ can shatter $\Omega\left(\frac{1}{\epsilon^2}\right)$ examples with margin $\epsilon$.
\end{lemma}
\begin{proof}
    Let $\colsnum=\frac{1}{\epsilon^2}$ and $W_0=\mathbf{0}_{\rowsnum\times \colsnum}$.
    Now, for every possible labeling for S, $y\in \{\pm \epsilon\}^{\frac{1}{\epsilon^2}}$, we define
    the matrix $W_y$ to be the matrix which its first row is $u_y$ and the rest of the rows are $0$. Note that $\frobnorm{W_y-W_0}=1$. Moreover, we define $\ell:\R^\rowsnum\to \R$ as
    $\ell(\hat{y})=e_1\hat{y}$. 
    This function is convex and $1$-Lipschitz. For every $i\in \left[\frac{1}{\epsilon^2}\right]$ we define $x_i=e_i$. It is left to show that the set $S=\{x_1\ldots,x_{\frac{1}{\epsilon^2}}\}$ can be shattered.
   It holds since for every $y$,
   \[\ell(W_yx_i)=e_1W_yx_i=ye_i=y_i.\]
\end{proof}

\begin{lemma}
\label{lem:set_direc_exists}
Let $d \geq 100$.
There exists a set $U_{d} \subseteq \R^{d}$, with 
$\abs{U_{d}} \geq 2^{d/12}$, such that all $u\in U_{d}$ are of unit length $\norm{u}=1$, and
for all $u, v\in U_{d}, u\neq v$, it holds that $\langle u, v\rangle  \leq \frac{1}{2}$.
\end{lemma}
\begin{proof}[of \cref{lem:set_direc_exists}]
Let $r=2^{\frac{d}{12}}$.
 For every $1\leq i\leq r$ and $1\leq j\leq d$ let the $u_i^j$ the random variable which is $\frac{1}{\sqrt{d}}$ with probability $\frac{1}{2}$ and $-\frac{1}{\sqrt{d}}$ with probability $\frac{1}{2}$. Then, for every $1\leq i\leq r$, we define the vector $u_i=(u_i^1, \cdots, u_i^{d})$ and look at the set $U=\{u_1,u_2,...u_{r}\}$. We now show that U satisfies the required property with positive probability.
    By Hoeffding's inequality, it holds that,
    \begin{align*}
        Pr(\langle u_i,u_k\rangle \geq \frac{1}{2})&\leq e^{\frac{-2\left(\frac{1}{2}\right)^2}{d\cdot \frac{4}{d^2}}}=e^{-\frac{d}{8}}.
    \end{align*}
Then, by union bound on the $\binom{r}{2}$ pairs of vectors in $U$,
\begin{align*}
        Pr(\exists i,k \ \langle u_i,u_k\rangle\geq \frac{1}{2})&\leq e^{-\frac{d}{8}}\cdot \binom{r}{2}< e^{-\frac{d}{8}}\cdot \frac{1}{2}r^2\leq 1.
    \end{align*}
\end{proof}

\begin{proof}[of \cref{main_thm_erm_lower}]
We first prove the theorem for the case of $\rowsnum= \Omega(\log \numdata)$.
Let $\epsilon$ which satisfy the condition of \cref{lower_bound_shattering}.
Let $\colsnum_0$, $W_0$ and $\ell$ be as defined in \cref{lower_bound_shattering}. By \cref{lower_bound_shattering}, there exists a constant $C$ and a set of examples $S=\{x_i\}_{i=1}^\frac{C\rowsnum}{\epsilon^2}$
such that for every labeling $y\in \{0,1\}^\frac{C\rowsnum}{\epsilon^2}$, there exists a matrix $W_y$ with $\frobnorm{W_y-W_0}\leq 1$ such that for every $i\in [\frac{C\rowsnum}{\epsilon^2}]$, $\ell(W_yx_i)=\epsilon$ if $y_i=1$ and $\ell(W_yx_i)=-\epsilon$ if $y_i=0$.
Now, let $y^*=\{0\}^\frac{C\rowsnum}{\epsilon^2}$ and $\optmat$ be the corresponding $W_{y^*}$ and let $D'$ be the uniform distribution over $S$. We prove that for every data set $S'$ such that $|S'|\leq \frac{C\rowsnum}{2\epsilon^2}$ sampled i.i.d from $D'$, there exists an ERM $\ermmat$ with $L(\ermmat)-L(\optmat)\geq \epsilon$, this will prove \cref{main_thm_erm_lower}.
Let $S'=\{x_{i_{1}}\ldots x_{i_{|S'|}}\}$ be such a data set. Let $y^S\in \{0,1\}^\frac{C\rowsnum}{\epsilon^2}$ be a labeling as following
\[y^S=\left\{\begin{array}{cc}
    0 & x_i\in S \\
     1 & x_i \notin S,
\end{array}\right.\]
and $W_S:=W_{y^S}$ be the corresponding matrix.
First, by the definition of $W_{y^S}$ it follows that $W_S$ is a ERM since it holds that
\begin{align*}
\hat L(W_S)=\frac{1}{|S'|}\sum_{j=1}^{|S'|}\ell(W_Sx_{i_{j}})=\frac{1}{|S|}\sum_{j=1}^{|S'|}-\epsilon=-\epsilon.
\end{align*}
Moreover, since at least $\frac{|S|}{2}$ of the examples in $S$ are not in $S'$, it also holds that
\begin{align*}
    L(W_S)-L(\optmat)&=
    \frac{1}{|S|}\sum_{i=1}^{|S|} \ell(W_Sx_i)-\ell(\optmat x_i)
    \\&= \frac{1}{|S|} \sum_{i=1}^{|S|} \ell(W_Sx_i)+\epsilon
    \\&\geq \frac{1}{|S|}  \sum_{i\notin S'}\ell(W_Sx_i)+\epsilon
    \\&\geq \frac{1}{2}  \cdot 2\epsilon
    \\&= \epsilon.
\end{align*}
The proof for the case of $\rowsnum= O(\log \numdata)$ is analogous and can be implied by using \cref{lower_bound_shattering_small_m} instead of \cref{lower_bound_shattering}.
\end{proof}
\section{Proofs of \cref{sec_reduct}}
\label{sec_reduct_proods}

\begin{proof}[of \cref{set_direction_exists_R_2}]
   Let $\phi:[a]\to \R^2$, $\phi(i)= \left(\sin\left(\frac{\pi i}{2a}\right),\cos\left(\frac{\pi i}{2a}\right)\right)^T$ and $\delta=1-\cos\left(\frac{\pi}{2a}\right)$.
    We notice that $0<\delta< 1$.
    Then,
as a result, for every $i$ it holds that
\begin{align*}   \|\phi(i)\|=\sqrt{\sin\left(\frac{\pi i}{2a}\right)^2+\cos\left(\frac{\pi i}{2a}\right)^2}=1,
\end{align*}
and if $i\neq j$, 
\begin{align*}
   \langle \phi(i),\phi(j)\rangle&=\sin\left(\frac{\pi i}{2a}\right)\sin\left(\frac{\pi j}{2a}\right) + \cos\left(\frac{\pi i}{2a}\right)\cos\left(\frac{\pi j}{2a}\right)
   \\&=\cos\left(\frac{\pi (i-j)}{2a}\right)
   \\&\leq \cos\left(\frac{\pi}{2a}\right) 
   \tag{$\cos$ is monotonic decreasing in $[0,\pi/2]$}
   \\&=1-\delta.
\end{align*}
\end{proof}

\end{document}